\documentclass[10pt,twocolumn,letterpaper]{article}

\usepackage{cvpr}              
\usepackage[accsupp]{axessibility} 

\usepackage[dvipsnames]{xcolor}

\usepackage{amsmath,amsfonts,bm}

\def\1{\bm{1}}

\def\rmI{{\mathbf{I}}}

\DeclareMathAlphabet{\mathsfit}{\encodingdefault}{\sfdefault}{m}{sl}
\SetMathAlphabet{\mathsfit}{bold}{\encodingdefault}{\sfdefault}{bx}{n}

\def\gL{{\mathcal{L}}}

\def\gN{{\mathcal{N}}}

\def\gU{{\mathcal{U}}}

\DeclareMathOperator*{\argmin}{arg\,min}

\newcommand{\bbx}{\mathbf{x}}
\newcommand{\bbe}{\boldsymbol{\epsilon}}
\newcommand{\bbt}{\boldsymbol{\theta}}
\usepackage{amsthm}

\usepackage[ruled]{algorithm2e}

\definecolor{cvprblue}{rgb}{0.21,0.49,0.74}
\usepackage[pagebackref,breaklinks,colorlinks,citecolor=cvprblue]{hyperref}

\usepackage[toc,page]{appendix}

\newtheorem{theorem}{Theorem}

\theoremstyle{remark}
\newtheorem*{remark}{Remark}

\def\paperID{3392} 
\def\confName{CVPR}
\def\confYear{2024}

\title{Consistent3D: Towards Consistent High-Fidelity Text-to-3D Generation with Deterministic Sampling Prior}
\author{%
Zike Wu$^{1,4}$ \quad 
Pan Zhou$^{*2,4}$ \quad Xuanyu Yi$^{1,4}$ \quad Xiaoding Yuan$^{3}$ \quad Hanwang Zhang$^{1,5}$ \\
{\small $^1$Nanyang Technological University} \ 
{\small $^2$Singapore Management University} \ 
{\small $^3$Johns Hopkins University} \
{\small $^4$Sea AI Lab} \ 
{\small $^5$Skywork AI} \\
\footnotesize{\texttt{zike001@e.ntu.edu.sg}}, \quad \footnotesize{\texttt{panzhou@smu.edu.sg}}, \quad
\footnotesize{\texttt{xuanyu001@e.ntu.edu.sg}}, \quad
\footnotesize{\texttt{xyuan19@jhu.edu}}, \quad
\footnotesize{\texttt{hanwangzhang@ntu.edu.sg}} \\
}

\begin{document}
\twocolumn[{
\renewcommand\twocolumn[1][]{#1}
\maketitle
\centering
\vspace{-0.5cm}
\includegraphics[width=\linewidth]{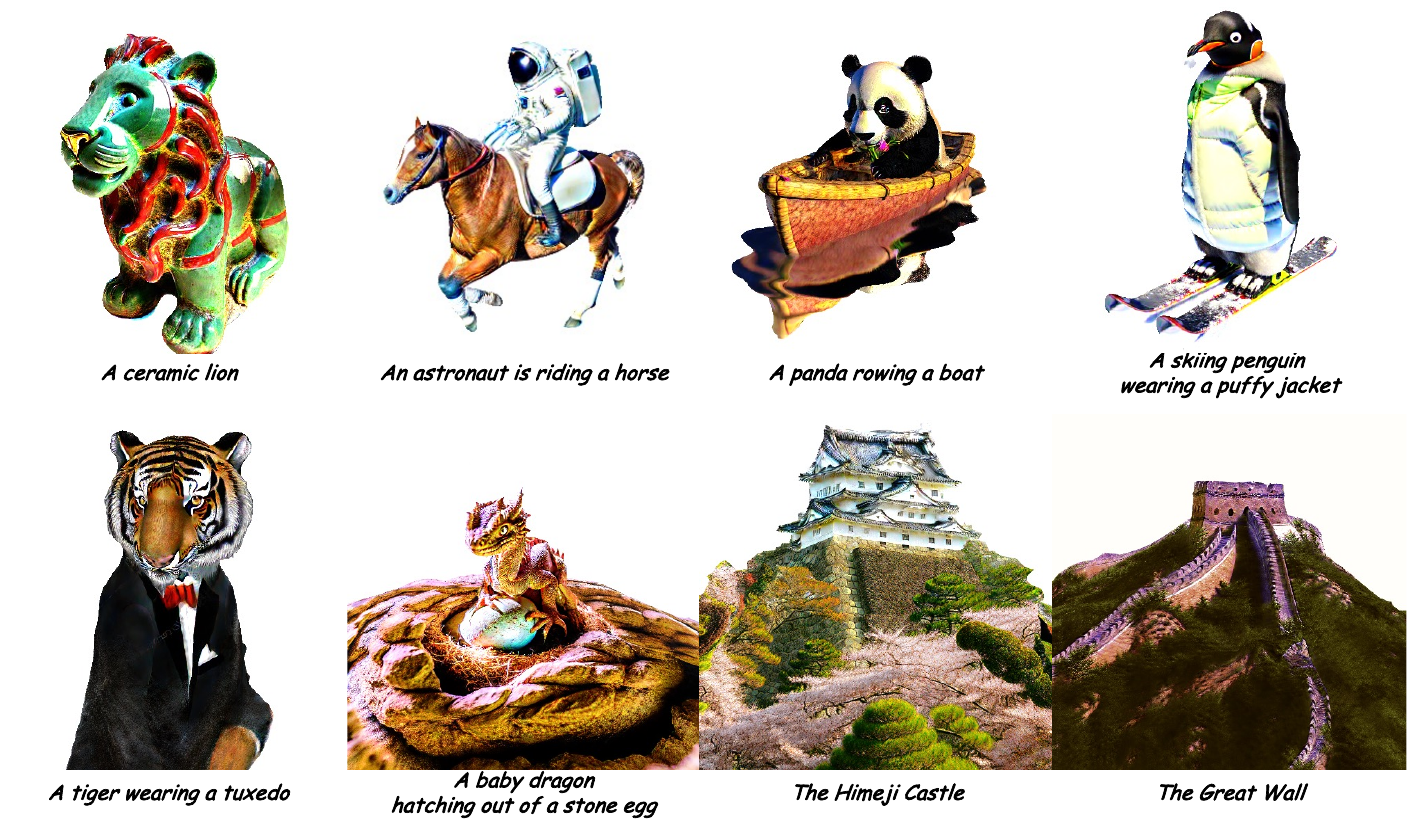}
\captionsetup{type=figure}
\vspace{-23pt}
\caption{Examples generated by Consistent3D. Our methods can generate detailed, diverse 3D objects and large-scale scenes from a wide range of textual prompts.}
\vspace{0.1cm}
\label{fig:1}
}]

{\let\thefootnote\relax\footnotetext{$^*$Corresponding author.}}

\begin{abstract}

Score distillation sampling (SDS) and its variants have greatly boosted the development of text-to-3D generation, but are vulnerable to geometry collapse and poor textures yet. To solve this issue, we first deeply analyze the SDS and find that its distillation sampling process indeed corresponds to the trajectory sampling of a stochastic differential equation~(SDE): SDS samples along an SDE trajectory to yield a less noisy sample which then serves as a guidance to optimize a 3D model. However, the randomness in SDE sampling often leads to a diverse and unpredictable sample which is not always less noisy, and thus is not a consistently correct guidance, explaining the vulnerability of SDS. Since for any SDE, there always exists an ordinary differential equation (ODE) whose trajectory sampling can deterministically and consistently converge to the desired target point as the SDE,  we propose a novel and effective ``Consistent3D" method that explores the ODE deterministic sampling prior for text-to-3D generation. Specifically, at each training iteration, given a rendered image by a 3D model, we first estimate its desired 3D score function by a pre-trained 2D diffusion model, and build an ODE for trajectory sampling. Next, we design a consistency distillation sampling loss which samples along the ODE trajectory to generate two adjacent samples and uses the less noisy sample to guide another more noisy one for distilling the deterministic prior into the 3D model. Experimental results show the efficacy of our Consistent3D in generating high-fidelity and diverse 3D objects and large-scale scenes, as shown in Fig.~\ref{fig:1}. The codes are available at \url{https://github.com/sail-sg/Consistent3D}.

\end{abstract}
\section{Introduction}
Diffusion models~(DMs) have recently garnered significant attention in the realm of image synthesis, as evidenced by their remarkable capabilities~\cite{rombach2022high,wu2023fast}. This notable progress can be largely attributed to the integration of large-scale image-text pair datasets and the evolution of scalable generative model architectures~\cite{ronneberger2015u,peebles2023scalable}. This recent success has seamlessly transcended into the domain of text-to-3D generation by leveraging the pre-trained 2D diffusion models~\cite{rombach2022high, ramesh2022hierarchical} to guide the 3D generation process, regardless of the absence of large-scale 3D generative models~\cite{wang2023score,chen2023fantasia3d}.

The pivotal breakthrough in this field stems from the finding that one can use the score function predicted by pre-trained 2D diffusion models, such as Stable Diffusion~\cite{rombach2022high}, to estimate the 3D score function~\cite{hong2023debiasing,wang2023score,graikos2022diffusion}. Since this score function indicates the direction of the higher data density~\cite{song2020improved,ho2022classifier}, one can first use it to build a stochastic differential equation~(SDE)~\cite{SGM}, and then sample along the SDE solution trajectory (\ie, SDE reverse process) to iteratively improve a learnable 3D model (\eg, NeRF~\cite{mildenhall2021nerf} or Mesh~\cite{shen2021deep}). This is also the underlying mechanism behind the prevalent and leading text-to-3D approach, Score Distillation Sampling~(SDS)~\cite{poole2022dreamfusion}. In each training iteration, SDS follows the forward SDE to inject noise into a rendered image by a learnable 3D model, and then samples a more realistic pseudo-image along the SDE solution trajectory, where the 3D score function of the SDE is estimated by a pre-trained diffusion model~\cite{wang2023score}. Next, SDS pulls its rendered image closer to the pseudo-image via optimizing the learnable 3D model.  

However, as illustrated in Fig.~\ref{fig:flow}, the high randomness inherent in the SDE solution distribution~\cite{SGM,DDPM} leads to a highly diverse and unpredictable next point in the solution trajectory, \eg, the pseudo-image~\cite{shi2023mvdream, zhu2023hifa} in SDS. Although this trajectory may eventually converge to a specific target, \eg, the desired realistic image in SDS, the sampled next point does not always provide the correct guidance in each iteration. This lack of reliability also applies to SDS, significantly increasing the optimization difficulty of the 3D model. It also helps to explain why SDS is so vulnerable and often suffers from geometry collapse and poor fine-grained texture in practice~\cite{wang2023prolificdreamer,shi2023mvdream,tang2023dreamgaussian}.  

\begin{figure}[!t]
\centering
\includegraphics[width=.9\linewidth]{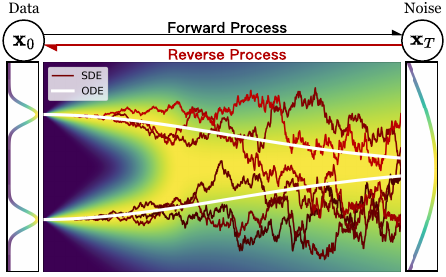}
\vspace{-6pt}
\caption{Comparison between the (reverse) trajectory samplings in the stochastic differential equation (SDE)  and ordinary differential equation (ODE).}
\label{fig:flow}
\vspace{-15pt}
\end{figure}

To address this critical issue, in this paper, we propose a novel and effective method, dubbed``Consistent3D'', which guides text-to-3D generation using deterministic sampling prior. As illustrated in Fig.~\ref{fig:flow}, for these unpredictable and uncontrollable trajectories sampled from SDE solution distribution, there theoretically always exists a corresponding ordinary differential equation (ODE) whose trajectory shares the same marginal distributions with the SDE solution~\cite{SGM}. Importantly, this ODE trajectory is deterministic and consistently converges to the same target point as the SDE. Sampling along the ODE trajectory guarantees a predictable and deterministic next point, which always directs towards the desired target and thus providing a reliable and consistent guidance. This motivates us to explore the text-to-3D generation from the ODE deterministic sampling perspective. 

Specifically, during each training iteration, we begin by estimating the desired 3D score function from the rendered images produced by the learnable 3D model using pre-trained 2D diffusion models. Subsequently, we build a corresponding ODE for solution trajectory sampling. To effectively optimize the underlying 3D representations, we then introduce a Consistency Distillation Sampling loss~(CDS), which leverages deterministic sampling prior along the ODE flow. In detail, for each rendered image, we first inject a \textit{fixed} noise to the rendered image so that the corresponding noisy sample lies in the ODE solution distribution and thus can be well denoised. Following this, we sample two \textit{adjacent} points from the ODE trajectory given the noisy sample. The less-noisy sample is then used to guide its more-noisy counterpart, thereby distilling the deterministic prior of the ODE trajectory into the 3D model. Here we use fixed noise to ensure that the samplings from the ODE trajectory for all rendered images converge to the same targeted realistic image, thereby offering more consistent guidance and enhancing the optimization of the 3D model.
 
Extensive experimental results showcase the efficacy of Consistent3D in generating high-fidelity and diverse 3D objects, along with large-scale scenes, as shown in Fig.~\ref{fig:1} and Fig.~\ref{fig:our_results}. Comparative evaluations against existing methods, including DreamFusion~\cite{poole2022dreamfusion}, Magic3D~\cite{lin2023magic3d} and ProlificDreamer~\cite{wang2023prolificdreamer}, demonstrate the superiority of Consistent3D in both qualitative and quantitative terms. The proposed approach effectively addresses the challenges associated with randomness in the SDE solution distribution, offering a more reliable and consistent framework to guide the text-to-3D generation process.

\section{Related Works}
\noindent\textbf{Diffusion Models}~\cite{SGM,DDPM,yang2022diffusion} are powerful tools for complex data modeling and generation.
Their robust and stable capabilities for complex data modeling have also led to their successful application in various domains, such as image~\cite{dhariwal2021diffusion,gu2022vector,wu2023fast}, video~\cite{ho2022imagen,ho2022video,khachatryan2023text2video}, and 3D~\cite{poole2022dreamfusion, wang2023score}, \etc. 
Regarding improving the sampling efficiency, there are two main approaches: learning-free sampling and learning-based sampling. Learning-free sampling typically involves discretizing reverse-time SDE~\cite{SGM,CLD} or ODE~\cite{lu2022dpm,DDIM,edm,PSNR,gDDIM}, while learning-based sampling is mainly based on knowledge distillation~\cite{meng2022distillation,PD,CD}. This paper is driven by recent progress in learning-based sampling, particularly in distilling knowledge from ODE sampling~\cite{PD,CD}.

\noindent\textbf{Text-to-3D Generation} stands for generating 3D contents from a given text description. Current 3D generative models~\cite{nichol2022point,jun2023shap}, usually work in a single object category and suffer from limited diversity due to the lack of large-scale 3D datasets. To achieve open-vocabulary 3D generation, pioneered by DreamFusion~\cite{poole2022dreamfusion}, several approaches propose to lift text-image diffusion models~\cite{rombach2022high} for 3D generation~\cite{zhao2023efficientdreamer,wang2023prolificdreamer,zhu2023hifa}. The key mechanism of such approaches is the score distillation sampling (SDS), where diffusion priors are used to supervise the optimization of a 3D representation. The following works continue to further improve the stability and fidelity of generation of various aspects, \eg, advanced 3D representation~\cite{chen2023fantasia3d,tang2023dreamgaussian,yi2023gaussiandreamer,wang2023neus2}, coarse-to-fine training strategy~\cite{lin2023magic3d,zhu2023hifa,wang2023prolificdreamer,yi2024diffusion} and 3D-aware diffusion priors~\cite{zhao2023efficientdreamer,shi2023mvdream,li2023sweetdreamer,long2023wonder3d,tewari2023diffusion}. 

\section{Preliminaries}
\label{sec:pre}

\noindent\textbf{Diffusion Models (DMs).}
They consist of a forward diffusion process and a reverse sampling process. During the forward process, DMs gradually add Gaussian noise to the vanilla sample $\bbx_0 \sim p_{\text{data}}(\bbx)$ and generate a series of noisy samples $\bbx_t$ according to the distribution:
\begin{equation}
	\label{eq:kernel}
	p_t(\bbx_t | \bbx_0) = \gN(\bbx_t; \bbx_0, \sigma_t^2 \rmI),
\end{equation}
where $\sigma_t$ varies along time-step $t$. Accordingly, one can easily sample a noisy sample at any time step $t$ by $\bbx_t = \bbx_0 + \sigma_t \boldsymbol{\epsilon}_t$, where $\boldsymbol{\epsilon}_t \sim \gN(\mathbf{0}, \rmI)$.

The reverse process from a Gaussian noise $\bbx_T$  to a realistic sample $\bbx_0$  is also called trajectory sampling,  and can be formally formulated into a reverse SDE \cite{edm}:
\begin{equation}
	\label{eq:sample_SDE}
	\mathrm{d}\bbx = -\dot{\sigma}_t \sigma_t \nabla \log p_t(\bbx) \mathrm{d}t + \sqrt{\dot{\sigma}_t \sigma_t} \mathrm{d}\mathbf{w},
\end{equation}
where $\mathbf{w}$ is the standard Wiener process, $\dot{\sigma}_t$ is the time derivative of ${\sigma}_t$, and $\nabla \log p_t(\bbx)$ is the \textit{score function} which indicates the direction of the higher data density~\cite{SGM}. Meanwhile, there exists a corresponding reverse ordinary deterministic equation (ODE) which is defined as follows:
\begin{equation}
	\label{eq:sample_ODE}
	\mathrm{d}\mathbf{x} = -\dot{\sigma}_t \sigma_t \nabla \log p_t(\bbx) \mathrm{d}t.
\end{equation}
For this ODE, its trajectory shares the same marginal probability density as the SDE which ensures the same convergence point of ODE and SDE.  

Given any noise $\bbx_T \sim \mathcal{N}(\mathbf{0}, \sigma_T^2 \mathbf{I})$, one can either solve the reverse SDE in Eq.~\eqref{eq:sample_SDE} or the reverse ODE in Eq.~\eqref{eq:sample_ODE} via any numerical solver~\cite{atkinson1991introduction,lu2022dpm,zhang2023improved} to generate a real sample $\hat\bbx_0 \sim p_{\text{data}}(\mathbf{x})$. In practice, the pre-trained diffusion models~\cite{edm,rombach2022high} are used to estimate the score function, thereby guiding the sampling process.

\noindent\textbf{Text-to-3D Generation via Score Distillation Sampling (SDS).}
Given a camera pose $\pi$, SDS distills 2D priors of a pre-trained diffusion model $D_\phi(\cdot)$ into a 3D model (\eg, NeRF, Mesh) parameterized by $\bbt$. Formally, SDS applies a denoising training objective to the rendered image $\bbx_{\pi} = g(\bbt, \pi)$ where $g(\cdot)$ is a differentiable renderer and $\pi$ is a camera pose, and computes the gradient as
\begin{equation}
	\label{eq:sds}
	\nabla_{\bbt} \mathcal{L}_{\text{SDS}}(\bbt) = 
	\mathbb{E}_{t, \boldsymbol{\epsilon}}\left[\lambda(t)\left(\bbx_{\pi} - D_{\phi}(\bbx_t, t, y) \right) \frac{\partial \bbx_{\pi}}{\partial \bbt}\right],
\end{equation}
where $\lambda(t)$ denotes the loss weight, $\bbx_t = \bbx_{\pi} + \sigma_t \boldsymbol{\epsilon}_t$ denotes the noisy sample and $y$ is the text condition. 
\section{Method}
Here we elaborate on our proposed Consistent3D for effective text-to-3D generation. In Sec.~\ref{subsec:SDS_SDE}, we reveal the underlying mechanism of Score Distillation Sampling~(SDS) which aims to approximate the SDE sampling process and motivates our proposed methods. Then in Sec.~\ref{subsec:CDS}, we introduce Consistency Distillation Sampling~(CDS), a loss designed to efficiently distill deterministic sampling priors for text-to-3D generation, accompanied with a theoretical justification regarding the error bound. Finally, we present how to use our CDS to build our text-to-3D generation framework, dubbed ``Consistent3D", in Sec.~\ref{subsec:application}.

\begin{algorithm}[t!]
\SetKwInOut{Input}{Input}\SetKwInOut{Output}{Output}
\SetKwComment{Comment}{// }{}
\SetKw{Initialize}{\textbf{Initialize}}
\SetKw{Sample}{\textbf{Sample}}
\LinesNotNumbered 
\SetSideCommentRight
\Input{initial 3D model parameter $\bbt$, pre-trained diffusion model $D_{\phi}$, text prompt $y$, training iteration $N$, time-step range $[t_{\min}, t_{\max}]$, learning rate $\eta$}
\Output{$\bbt$}

\Sample $\boldsymbol{\epsilon}^* \sim \gN(\mathbf{0}, \rmI)$ \tcp{Fixed noise}

\ForEach{$i \in \{0, \dots, N\}$}{
    $t_2 \gets t_{\max} - (t_{\max} - t_{\min}) \sqrt{{i}/{N}}$

    {
        \Sample camera pose $\pi$
        
        \Sample $t_1 \in \gU[t_2 + \delta, t_2 + \Delta]$
    
        $\bbx_{\pi} \gets g(\bbt, \pi)$
        
        $\bbx_{t_1} \gets \bbx_{\pi} + \sigma_{t_1} \boldsymbol{\epsilon}^*$
        
        $\mathbf{d}_i \gets \left(\bbx_{t_1} - D_{\phi}(\bbx_{t_1}, t_1, y)\right)/\sigma_{t_1}$
        
        $\hat{\bbx}_{t_2} \gets \bbx_{t_1} + (\sigma_{t_2} - \sigma_{t_1})\mathbf{d}_i$
    
        $\hat{\bbx}_0 \gets \bbx_{\pi} + \operatorname{sg}\left[\sigma_{t_1} (\boldsymbol{\epsilon}^* - \mathbf{d}_i)\right]$
        
        $\gL_{\text{CDS}}(\bbt; \pi) \gets \lambda(t_2) \lVert \hat{\bbx}_0 - \operatorname{sg}\left[D_{\phi}(\hat{\bbx}_{t_2}, t_2, y)\right] \rVert_2^2$
        
        $\bbt \gets \bbt - \eta \nabla_{\bbt} \gL_{\text{CDS}}(\bbt; \pi)$
    }
}

\caption{Text-to-3D Generation with CDS}
\label{alg:CDS}
\end{algorithm}

\subsection{Revisit Score Distillation Sampling}
\label{subsec:SDS_SDE}
Before introducing our proposed method, we first connect SDE in Eq.~\eqref{eq:sample_SDE} with the leading text-to-3D generation approach SDS in Eq.~\eqref{eq:sds}, since this connection directly motivates us to use ODE in Eq.~\eqref{eq:sample_ODE} for text-to-3D generation.   

First, we discretize the reverse SDE in Eq.~\eqref{eq:sample_SDE} and perform the stochastic sampling process following \citet{DDPM}, which results in the SDE solution trajectory defined as
\begin{equation}
\label{eq:SDE_iter}
\begin{aligned}
    \bbx_{t_i} &= \hat{\bbx}^{i-1} + \sigma_{t_{i}} \bbe_{t_{i}} \quad \text{with} \quad \bbe_{t_{i}} \sim \gN(\mathbf{0}, \rmI), \\
    \hat{\bbx}^{i} &= D_{\phi}(\bbx_{t_i}, t_i),
\end{aligned}
\end{equation}
where $\hat{\bbx}^{0} = \mathbf{0}$ is to ensure $\hat{\bbx}_{T} = \hat{\bbx}^{0} + \sigma_{T} \bbe_{T} \sim \gN(\mathbf{0}, \sigma_T^2 \rmI)$, and time step schedule $\{t_i\}$ satisfies $T = t_1 > t_2 > \dots > t_N = 0$. To build a connection between SDE and SDS, in Eq.~\eqref{eq:SDE_iter}, we follow SDS to approximate the score function $\nabla \log p_t(\bbx)$ in vanilla SDE with a score network $D_{\phi}(\bbx_t, t)$. Then by iteratively running Eq.~\eqref{eq:SDE_iter} from $t_1$ to $t_N$, one can eventually compute the desired SDE solution in expectation, \eg, a real sample $\hat{\bbx}_N \sim p_{\text{data}}(\bbx)$ if the network $D_{\phi}(\cdot)$ is a well-trained diffusion model like Stable Diffusion~\cite{rombach2022high}.  

On the other hand, by fixing the camera pose $\pi$, by fixing the camera pose $\pi$, for a rendered image $\bbx_\pi$ by a learnable 3D model $\bbt$, the optimization process of SDS introduced in Sec.~\ref{sec:pre} can be formulated as: 
\begin{equation}
\label{eq:SDS_iter}
\small
\begin{aligned}
    \bbx_{t_{i}} &\!=\! \bbx_{\pi}^{i-1} \!+\! \sigma_{t_{i}} \bbe_{t_{i}}  \ \text{with} \  \bbe_{t_{i}} \sim \gN(\mathbf{0}, \rmI), \\
    \bbx_{\pi}^{i} &\!=\! g(\bbt^{i}, \pi)  \ \text{with}  \   \bbt^{i} \!=\! \argmin_{\bbt} \| g(\bbt, \pi) \! -\! D_{\phi}(\bbx_{t_i}, t_i)\|, \\
\end{aligned}
\end{equation}
where $\bbx_{\pi}^0 = g(\bbt^0, \pi)$ in which $\bbt^0$ denotes the randomly initialized 3D model~\cite{poole2022dreamfusion,lin2023magic3d}, and $g(\cdot)$ is a differentiable renderer~\cite{mildenhall2021nerf,garbin2021fastnerf}. Compared the stochastic sampling process in Eq.~\eqref{eq:SDE_iter} with SDS process in Eq.~\eqref{eq:SDS_iter}, one can observe that if for each iteration $i$, one can ideally optimize 3D model $\bbt^i$ so that $\| g(\bbt^{i}, \pi) - D_{\phi}(\bbx_{t_i}, t_i) \| = 0$, then one can have 
\begin{equation}
	\label{SDS_itedaasr}
	\begin{aligned}
	\bbx_{\pi}^{i} = g(\bbt^{i}, \pi)  = D_{\phi}(\bbx_{t_i}, t_i).
	\end{aligned}
\end{equation}
In this case, the SDS optimization process becomes exactly the same as the stochastic sampling process with $ \bbx_{\pi}^{i}$ replaced by $\hat\bbx^{i}$.

However, as illustrated in Fig.~\ref{fig:flow}, sampling along the SDE solution trajectory according to Eq.~\eqref{eq:SDE_iter} results in an unpredictable and highly variable next point $\hat{\bbx}^{i}$, which does not guarantee the correct direction. This issue also extends to the SDS optimization process, which is equivalent to the SDE trajectory in Eq.~\eqref{eq:SDS_iter} when the 3D model is ideally trained in each iteration (\ie, Eq.~\eqref{SDS_itedaasr} holds). 
Consequently, such inherent randomness in SDS leads to less accurate and reliable guidance throughout all training iterations. This could also help explain why SDS is so vulnerable and often suffers from geometry collapse and poor fine-grained texture as observed in many works \cite{tang2023dreamgaussian,wang2023prolificdreamer,zhu2023hifa}. 

\subsection{Consistency Distillation Sampling}
\label{subsec:CDS}
\noindent\textbf{3D Deterministic Sampling.}
Given the stochastic and unpredictable nature of SDS, we are motivated to explore the potential of the ODE deterministic process which can provide consistent and more accurate guidance than SDE for 3D generation as shown by Fig.~\ref{fig:flow}.
We start by focusing on the ODE sampling process for a 3D model $\bbt$:
\begin{equation}
\label{eq:3D_ODE}
	\mathrm{d}\bbt = -\dot{\sigma}_t \sigma_t \nabla \log p_t(\bbt) \mathrm{d}t,
\end{equation}
where $\bbt$ is randomly initialized according to a certain distribution. Following \citet{poole2022dreamfusion} and \citet{wang2023score}, one can derive the 3D score function $\nabla_{\bbt} \log p_t(\bbt) $ from the 2D score function using the chain rule:
\begin{equation}
\label{eq:chain}
    \nabla_{\bbt} \log p_t(\bbt) = \mathbb{E}_{\pi} \left[\nabla_{\bbx_{\pi}} \log p_t(\bbx_{\pi}) \frac{\partial \bbx_{\pi}}{\partial \bbt} \right],
\end{equation}
where the 2D score function  $\nabla_{\bbx} \log p_t(\bbx)$ can be estimated as $\nabla_{\bbx} \log p_t(\bbx) = (D_{\phi}(\bbx, t) - \bbx) / \sigma_t^2$ by a pre-trained diffusion model $D_{\phi}(\bbx, t)$. 
Therefore, the key to generating a satisfactory 3D model is to accurately perform the 3D ODE sampling in Eq.~\eqref{eq:3D_ODE} using the pre-trained diffusion model. 

Unfortunately, unlike the forward SDE process in which a noisy sample can be easily sampled from the perturbation kernel by $\bbx_t \sim p_t(\bbx_t | \bbx_{\pi})$ in Eq.~\eqref{eq:kernel}, the forward ODE requires iterative simulation of the ODE flow, such as DDIM inversion~\cite{mokady2023null} which is complex and time-consuming. This makes the approximation of the ODE flow with conventional SDS less efficient and is often impractical. Thus, directly applying SDS loss to the ODE flow is practically prohibited.  

Inspired by recent advances in diffusion model distillation techniques that facilitate approximation of this deterministic flow without extensive simulation~\cite{CD}, we develop a simple yet effective Consistency Distillation Sampling loss (CDS) tailored for general text-to-3D generation tasks. Further detailed discussions can be found in \cref{appsec:connection}.

\begin{figure}[!t]
\centering
\includegraphics[width=\linewidth]{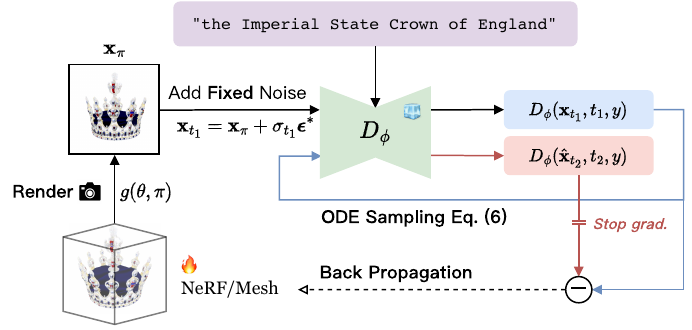}
\vspace{-15pt}
\caption{Overview of CDS. In each training iteration, the rendered image is perturbed by a fixed noise and then served as a start point of the deterministic flow for computing the CDS loss.}
\label{fig:framework}
\vspace{-15pt}
\end{figure}

\noindent\textbf{Optimization objective.}
We aim to enforce the optimization process of the 3D model $\bbt$ to match the deterministic flow between two \textit{adjacent} ODE sampling steps. Specifically, we always use a \textit{fixed} Gaussian noise $\bbe^*$ to perturb the sample, analogous to setting a fixed starting point at the final diffusion time step. This approach ensures a consistent perturbation in all iterations, similar to the technique used in Consistency Training~\cite{CD}. Next, we optimize $\bbt$ by minimizing  the following Consistency Distillation Sampling~(CDS) loss:
\begin{equation}
\label{eq:cds}
\mathbb{E}_{\pi} \left[\lambda(t_2) \lVert D_{\phi}(\bbx_{t_1}, t_1, y) - \operatorname{sg}(D_{\phi}(\hat{\bbx}_{t_2}, t_2, y)) \rVert_2^2\right],
\end{equation}
where $\operatorname{sg}(\cdot)$ is a stop-gradient operator, $t_1 > t_2$ are two adjacent diffusion time steps, $\bbx_{t_1} = \bbx_{\pi} + \sigma_{t_1} \bbe^*$, and $\hat{\bbx}_{t_2}$ is a less noisy sample derived from deterministic sampling by running one discretization step of a numerical ODE solver from $\bbx_{t_1}$. Particularly, we adopt the Euler solver to compute $\hat{\bbx}_{t_2}$ by:
\begin{equation}
\label{eq:step}
    \hat{\bbx}_{t_2} = \bbx_{t_1} + \frac{\sigma_{t_2} - \sigma_{t_1}}{\sigma_{t_1}}(\bbx_{t_1} - D_{\phi}(\bbx_{t_1}, t_1, y)).
\end{equation}
In practice, we follow \citet{poole2022dreamfusion} and reparameterize the first component in Eq.~\eqref{eq:cds} to skip the CDS gradient directly to $\bbx_{\pi}$ and $\bbt$ without computing the U-Net Jacobian. 

\noindent\textbf{Time step schedule.} As our target is to match the probability flow ODE of the reverse sampling process, we follow the conventional DMs~\cite{edm,SGM} and set the time steps to decrease monotonically along with the training iteration of the 3D models. This approach redefines our 3D generation process more as a deterministic sampling rather than a mere training process as previous SDS-based approaches~\cite{poole2022dreamfusion,wang2023score}, thus allowing us to take full advantage of deterministic sampling prior. 

Specifically, we define the time step schedule~\cite{yi2024diffusion} of $t_2$ according to the current training iteration:
\begin{equation}
\label{eq:anneal}
    t_2 := t_{\max} - (t_{\max} - t_{\min}) \sqrt{{i}/{N}},
\end{equation}
where $i$ and $N$ denotes the current iteration and total iteration, respectively. For the initial time step $t_1$ which indicates the perturbation level to the rendered image, we empirically uniformly sample it within $[t_2 + \delta, t_2 + \Delta]$, which is different from the predetermined time step schedule in Consistency Distillation~\cite{CD}. This is because we empirically find that the random sampled time step $t_1$ within a \textit{small} interval collaborated with the deterministic anchor $t_2$ exhibits self-calibration behaviors, which can actively correct the cumulative error made in earlier steps and alleviate issues such as floaters and Janus faces~\cite{hong2023debiasing,shi2023mvdream}. We delve deeper into this phenomenon in Sec.~\ref{subsec:ablation}. For more clarity, we summarize our entire text-to-3D generation procedure with the proposed CDS in Algorithm~\ref{alg:CDS}.

\noindent\textbf{Justification.} In the following, we offer a theoretical justification to demonstrate that, upon achieving convergence, our Consistency Distillation Sampling is capable of generating a high-fidelity 3D model.

\begin{theorem}
\label{thm:1}
Assume that the diffusion model $D_{\phi}(\cdot)$ satisfies the Lipschitz condition. Define $\Delta := \sup |t_1 - t_2| $.
For any given camera pose $\pi$, if convergence is achieved according to Eq.~\eqref{eq:cds}, then there exists a corresponding real image $\bbx^* \sim p_{\text{data}}(\bbx)$ such that
\begin{equation}
\| \bbx_{\pi} - \bbx^* \|_2 = \mathcal{O}(\Delta),
\end{equation}
where $\bbx_{\pi} = g(\bbt, \pi)$ denotes the rendered image for pose $\pi$.
\end{theorem}
\begin{proof}
The proof is based on the truncation error of the Euler solver. We provide the full proof in \cref{appsec:thm}.
\end{proof}
For a 3D model optimized using the CDS, Theorem~\ref{thm:1} guarantees that images rendered from any viewpoint of this model are realistic and closely align with the corresponding real-world scenes.

\begin{figure*}[!t]
\centering
\includegraphics[width=\linewidth]{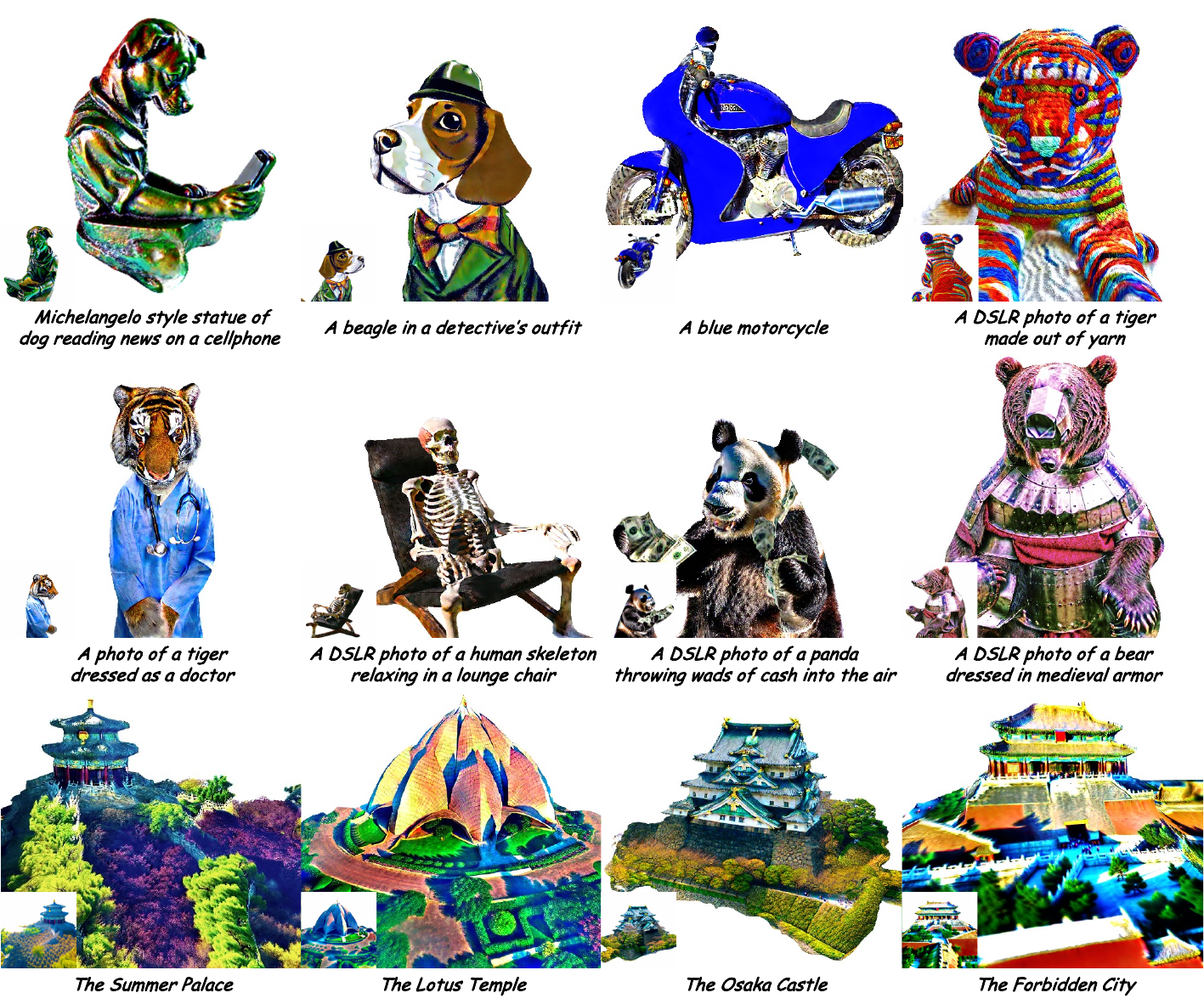}
\vspace{-18pt}
\caption{Consistent3D can generate diverse and high-fidelity objects or large-scale scenes highly correlated with the given text prompts.}
\label{fig:our_results}
\vspace{-12pt}
\end{figure*}

\begin{figure*}
\centering
\includegraphics[width=\linewidth]{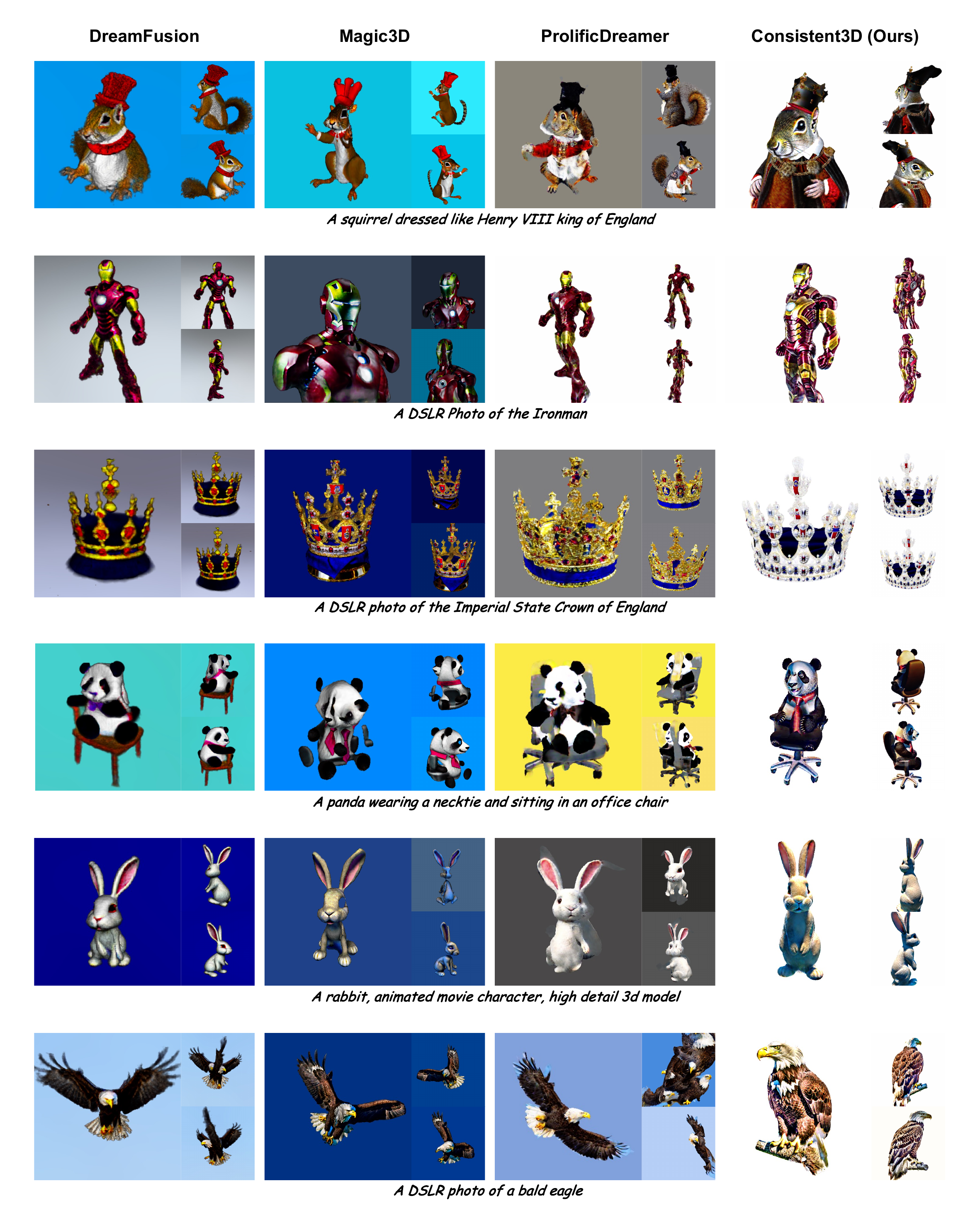}
\caption{\textbf{Qualitative Comparisons} of Text-to-3D Generation. Our approach yields results with enhanced fidelity and more robust geometry.}
\label{fig:comparison}
\end{figure*}

\subsection{Consistent3D}
\label{subsec:application}
Now we are ready to introduce our proposed Consistent3D. As illustrated in Fig.~\ref{fig:framework}, we present a clear design space for our Consistent3D generation framework using our proposed Consistency Distillation Sampling (CDS). 

Following previous work~\cite{lin2023magic3d,wang2023prolificdreamer}, Consistent3D is a coarse-to-fine approach consisting of two stages. Specifically, in the coarse stage, we optimize a low-resolution Neural Radiance Field (NeRF)~\cite{muller2022instant,barron2021mip}. For the refinement stage, we further optimize a high-resolution textured 3D mesh~\cite{shen2021deep} from the neural field initialization converting from the coarse stage. For these two stages, we always use our proposed CDS.

\noindent\textbf{NeRF Optimization Stage.} We adopt multi-resolution hash grids, Instant NGP~\cite{muller2022instant} to parameterize the scene by density and color with MLPs, which improves training and rendering efficiency. We follow Magic3D on density bias initialization, camera and light augmentation. In addition, we use orientation loss~\cite{poole2022dreamfusion} and 2D normal smooth loss~\cite{melas2023realfusion}. At this stage, we render $64 \times 64$ images and use our proposed CDS as guidance. We set $t_{\max} = 0.7$, $t_{\min} = 0.1$, $\delta = 0.1$, and $\Delta = 0.2$.

\noindent\textbf{Mesh Refinement Stage.} We convert the neural field into Signed Distance Field~(SDF) by subtracting it with a fixed threshold and then optimizing a high-resolution DMTet~\cite{shen2021deep}. We also initialize the volume texture field directly with the color field from the coarse stage. In addition, we use normal consistency loss and Laplacian smoothness loss. In the refinement stage, we render $512 \times 512$ images and set $t_{\max} = 0.5$, $t_{\min} = 0.02$, $\delta = 0.1$, and $\Delta = 0.1$.

\noindent\textbf{Fast Generation with 3D Gaussian Splatting.} Our Consistent3D with Consistency Distillation Sampling is a general text-to-3D generation framework that can be used to create a variety of 3D representations, including 3D Gaussian Splatting~\cite{kerbl20233d}. See results in \cref{appsubsec:gaussian}, our Consistent3D is capable of producing high-fidelity 3D models with intricate details in 15 minutes.

\section{Experiment}
\label{sec:exp}

\begin{figure*}[!t]
\centering
\includegraphics[width=\linewidth]{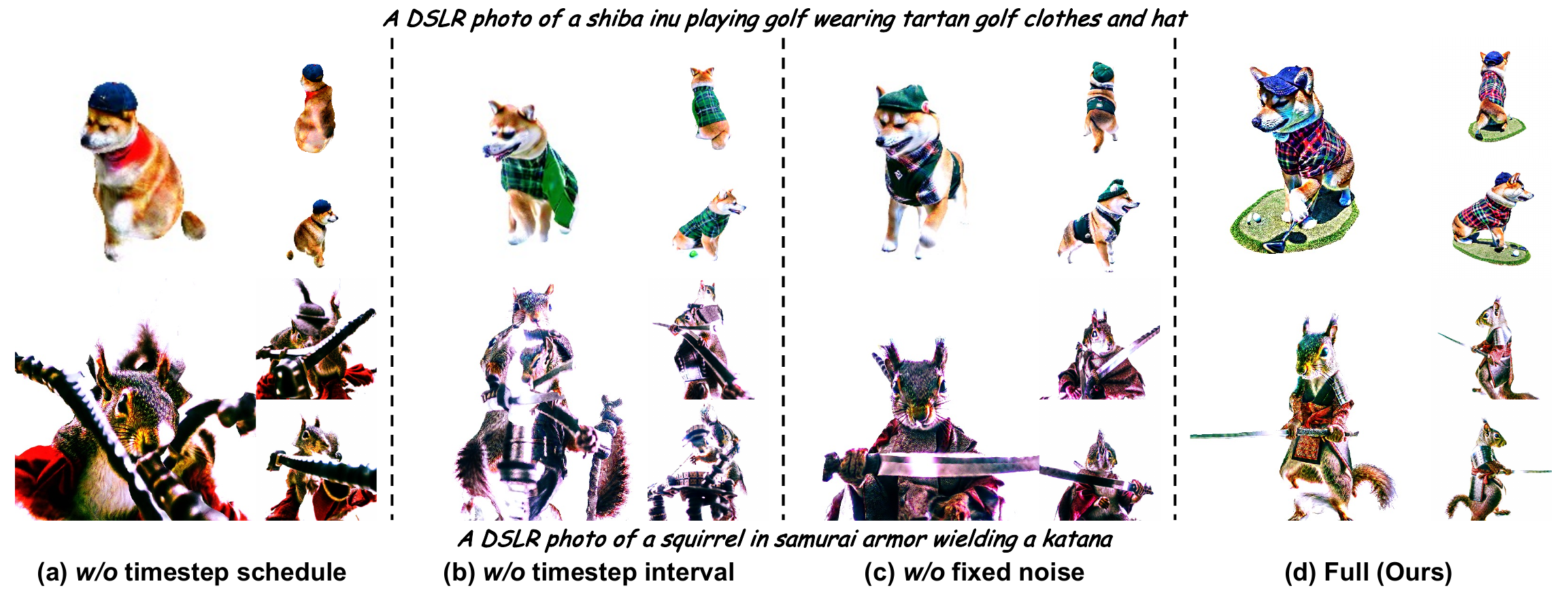}
\vspace{-25pt}
\caption{\textbf{Ablation study} of component-wise contribution of Consistent3D: (a) random time step schedule; (b) predetermined time step schedule; (c) random noise in each iteration; (d) our proposed configuration.}
\label{fig:ablation}
\vspace{-18pt}
\end{figure*}

\subsection{Implementation Details}
Consistent3D is implemented in PyTorch with a single NVIDIA A100 GPU based on \textit{threestudio}~\cite{threestudio2023} with Stable Diffusion v2.1~\cite{rombach2022high}. We use the Adan~\cite{xie2022adan} optimizer with a learning rate of $0.05$ for grid encoder and $0.005$ for other parameters, and a weight decay of $2\times10^{-8}$. Further implementation details are provided in \cref{appsubsec:implement}.

\subsection{Text-guided 3D Generation}
As illustrated in Fig.~\ref{fig:our_results}, our Consistent3D demonstrates versatility in generating high-fidelity 3D objects. Its generated images are not only realistic but also maintain consistency from various viewpoints. Furthermore, it is capable of generating large-scale scenes in $360^{\circ}$ with remarkable detail. See more qualitative results in \cref{appsubsec:qualitative}.

\subsection{Comparison with the State-of-The-Art}
\begin{table}[!t]
\centering
\begin{tabular}{ccc}
\toprule
Method & Loss & CLIP-R         \\
\midrule
DreamFusion~\cite{poole2022dreamfusion}         & SDS  & 0.310          \\
Magic3D~\cite{lin2023magic3d}             & SDS  & 0.311          \\
ProlificDreamer~\cite{wang2023prolificdreamer}     & VSD  & 0.336          \\
\hline
Consistent3D (Ours) & CDS  & \textbf{0.348} \\
\bottomrule
\end{tabular}
\vspace{-6pt}
\caption{\textbf{Quantitative Comparisons} of CLIP R-Precision. Scores were averaged from $40$ prompts in the DreamFusion gallery.}
\label{tab:CLIP}
\vspace{-15pt}
\end{table}
In this section, we present comprehensive qualitative and quantitative experiments to evaluate the efficacy of our Consistent3D framework in text-to-3D generation. We compare our generation performance with \textit{DreamFusion}~\cite{poole2022dreamfusion}, \textit{Magic3D}~\cite{lin2023magic3d}, and \textit{ProlificDreamer}~\cite{wang2023prolificdreamer}. For a fair comparison, we use the implementations of all the baseline methods from the open-source repository \textit{threestudio}~\cite{threestudio2023}.

\noindent \textbf{Qualitative Results.}
In Fig.~\ref{fig:comparison}, we provide qualitative comparisons with the baseline methods. Our approach exhibits more photorealistic details and geometry than both the SDS-based approaches like DreamFusion and Magic3D and the VSD-based approach ProlificDreamer. This improvement mainly stems from our Consistency Distillation Sampling~(CDS) which effectively leverages the full potential of large-scale diffusion models by accurately distilling deterministic sampling priors into the 3D model.

\noindent \textbf{Quantitative Results.}
In \cref{tab:CLIP}, we report the results of CLIP R-Precision~\cite{radford2021learning} for 3D objects generated using $40$ randomly selected text prompts from the DreamFusion gallery. See more details in \cref{appsubsec:implement}. Each 3D object is rendered from $120$ viewpoints with a uniform azimuth angle. The CLIP R-Precision score is computed by averaging the similarity scores between each rendered view and the corresponding text prompt~\cite{yi2023invariant}. Additionally, we also conduct a head-to-head user study in \cref{tab:study}. Our quantitative analysis shows the superior performance of our method. 

\subsection{Ablation Study}
\label{subsec:ablation}
We present an ablation study to evaluate the effects of various components in our approach in Fig.~\ref{fig:ablation} and \cref{tab:ab}. We conduct experiments with the following configurations: (a) a random time step schedule in DreamFusion~\cite{poole2022dreamfusion}; (b) a predetermined time step schedule from Consistency Distillation~\cite{CD}; (c) varied random noise in each iteration; and (d) our proposed method incorporating all components. The results  in  Fig.~\ref{fig:ablation}(a) reveal that a random time step schedule detrimentally affects both geometry and texture modeling, since it disrupts established rules of sampling process. Fig.~\ref{fig:ablation}(b) suggests that a predetermined time-step schedule is suboptimal for optimization-based methods, since gradient descent does not ensure monotonic optimization progress. This implies that minor randomness helps to accommodate these variations. Fig.~\ref{fig:ablation}(c) shows that fixed noise aids in better convergence by providing a consistent perturbation in each iteration.

\section{Conclusion}
In this work, we first connect Score Distillation Sampling (SDS), a leading text-to-3D generation approach, with the solution trajectory sampling of a stochastic differential equation (SDE). This connection helps us to understand the vulnerability in SDS, since the randomness in SDE sampling often provides a highly diverse sample, which is not always less noisy, and could guide the 3D model in the wrong direction. Then motivated by the fact that an ordinary differential equation (ODE) of an SDE can provide a deterministic and consistent sampling trajectory, we propose a novel and effective ``Consistent3D" by designing a consistency distillation sampling loss to distill the deterministic sampling prior into a 3D model for text-to-3D generation. Extensive experimental results show that our Consistent3D surpasses state-of-the-art methods in generating high-fidelity and diverse 3D objects and large-scale scenes.

\section*{Acknowledgement}
Pan Zhou was supported by the Singapore Ministry of Education (MOE) Academic Research Fund (AcRF) Tier 1 grant.

{
    \small
    \bibliographystyle{ieeenat_fullname}
    \bibliography{ref}
}

\clearpage
\maketitlesupplementary
\appendix


\section{Discussion}
\label{appsec:connection}
Diffusion models start by diffusing $p_{\text{data}}(\bbx)$ with a stochastic differential equation (SDE):
\begin{equation}
\mathrm{d} \mathbf{x} =\mu(t) \bbx \mathrm{d} t + \sigma(t) \mathrm{d} \mathbf{w},
\end{equation}
where $t \in [0, T]$, $\mu(\cdot)$ and $\sigma(\cdot)$ are the drift and diffusion coefficients respectively, and $\mathbf{w}$ denotes the standard Brownian motion. We denote the distribution of $\bbx_t$ by $p_t(\bbx)$.
A notable characteristic of this SDE is that there exists an Ordinary Differential Equation (ODE), named the Probability Flow (PF) ODE~\cite{SGM}, whose solution trajectories, when sampled at time $t$, adhere to the distribution $p_t(\bbx)$:
\begin{equation}
\mathrm{d} \mathbf{x} = \left[\mu(t)\bbx - \frac{1}{2} \sigma(t)^2 \nabla \log p_t(\mathbf{x})\right] \mathrm{d} t.
\end{equation}

Due to the above connection between the PF ODE and forward SDE, one can sample along the distribution of the ODE trajectories by first sampling $\bbx \sim p_{\text{data}}(\bbx)$, then adding Gaussian noise to $\bbx$. This implies that we can effectively sample two solutions on the PF ODE trajectory by first rendering an image $\bbx_{\pi}$, followed by perturbing it with Gaussian noise $\bbe^*$.
In this paper, we follow \citet{edm} and formulate the forward and reverse process as illustrated in \cref{sec:pre}, particularly $\mu(t) = 0$ and $\sigma(t) = \sqrt{2t}$. Thus, the perturbed sample is as follows:
\begin{equation}
    \bbx_{t_1} = \bbx_{\pi} + \sigma_{t_1} \bbe^*,
\end{equation}
and then computing $\hat{\bbx}_{t_2}$ using one discretization step of the numerical ODE solver by:
\begin{equation}
    \hat{\bbx}_{t_2} = \bbx_{t_1} + \frac{\sigma_{t_2} - \sigma_{t_1}}{\sigma_{t_1}}(\bbx_{t_1} - D_{\phi}(\bbx_{t_1}, t_1)).
\end{equation}
By optimizing the underlying 3D representation $\bbt$ to find an optimal $\bbx_{\pi}$, we can minimize the distance of the predicted sample from $D_{\phi}(\cdot)$ given $\bbx_{t_1}$ and $\hat{\bbx}_{t_2}$, 
\ie, minimizing the discrepancy of $D_{\phi}(\bbx_{t_1}, t_1)$ and $D_{\phi}(\hat{\bbx}_{t_2}, t_2)$ (\cref{eq:cds}). Therefore, we can eventually align $\bbx_{\pi}$, $\hat{\bbx}_{t_2}$ and $\bbx_{t_1}$ into the same deterministic flow, thus making $\bbx_{\pi}$ a realistic data point as it becomes a solution of the ODE sampling flow. Note that although this process introduces a truncation error from numerical ODE solvers, we will prove that our CDS achieves the same accuracy as multi-step sampling approaches in a single-step generative framework in \cref{appsec:thm}, and this error bound is almost optimal since one always needs to discretize the ODE flow to simulate it by diffusion models.

\section{Experiments}
\label{appsec:exp}

\begin{figure*}
\centering
\includegraphics[width=\textwidth]{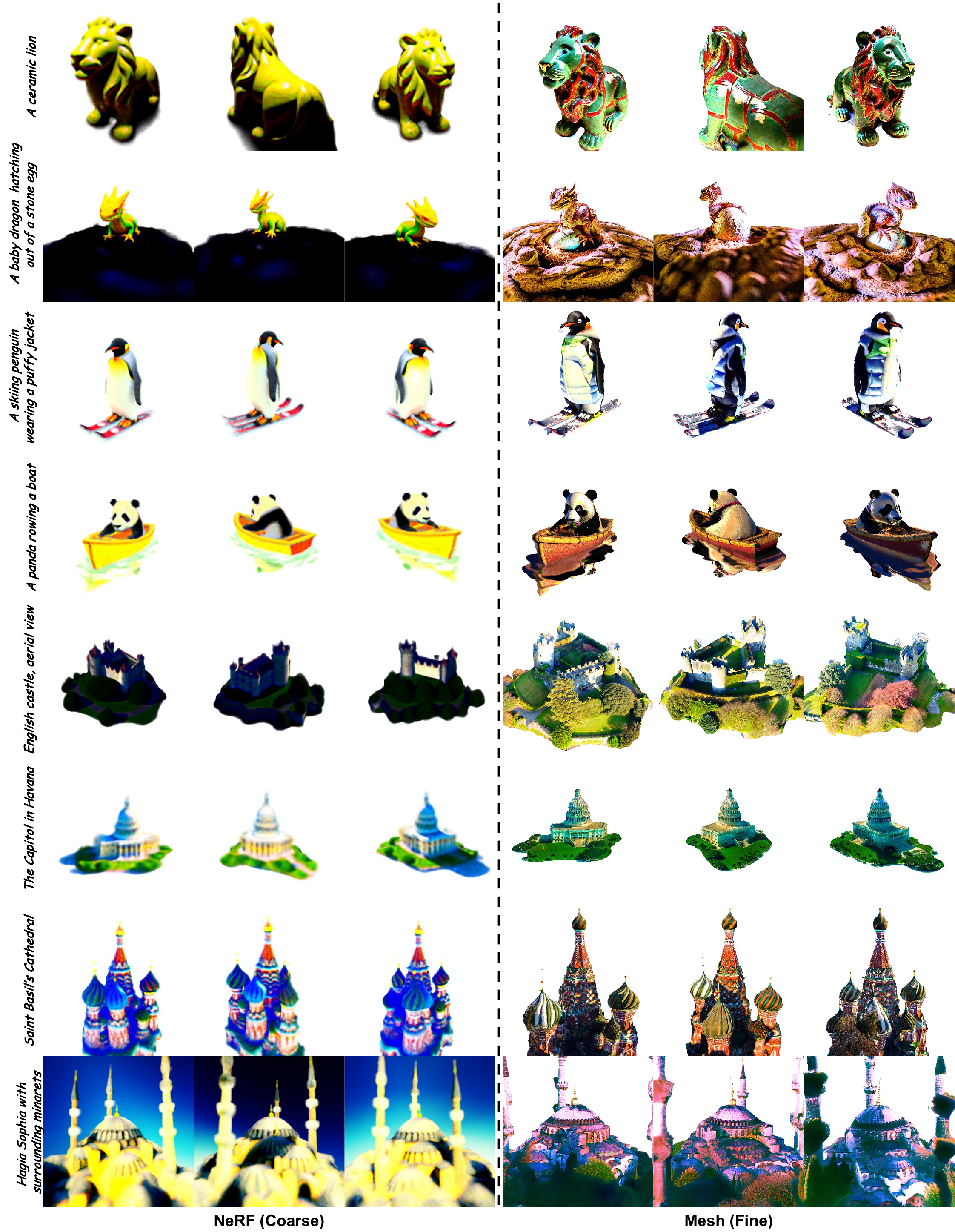}
\vspace{-18pt}
\caption{\textbf{Qualitative Results} of Two-Stage text-to-3D Generation: NeRF Optimization and Mesh Refinement.}
\label{fig:app-ours}
\end{figure*}

\begin{figure*}
\centering
\includegraphics[width=\textwidth]{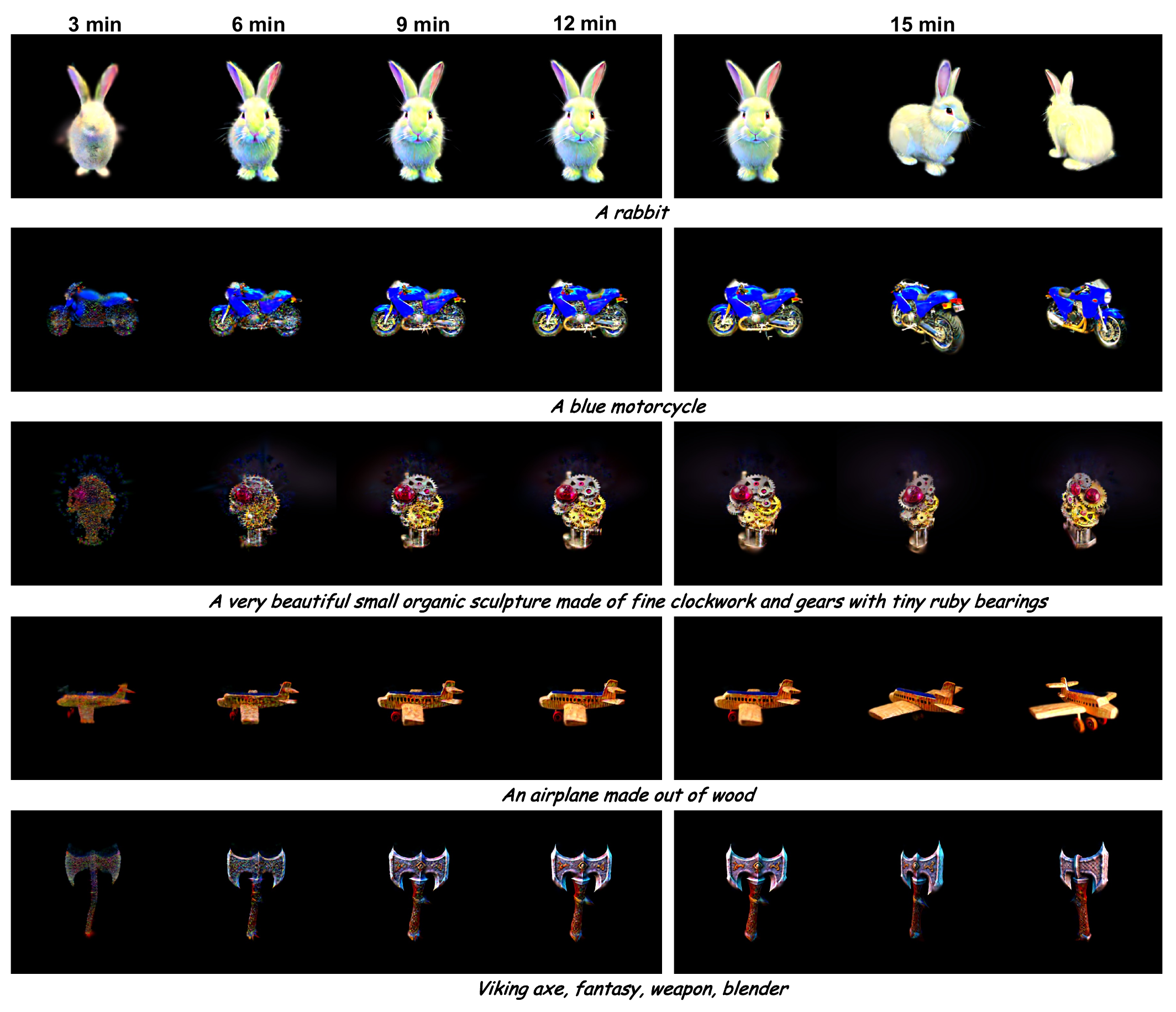}
\vspace{-6pt}
\caption{\textbf{Text-to-3D Generation using Consistent3D with 3D Gaussian Splatting.} Qualitative results demonstrate that our Consistent3D is capable of producing detailed, high-fidelity 3D models from text prompts in 15 minutes.}
\label{fig:app-gaussian}
\end{figure*}

\subsection{Additional Implementation Details}
\label{appsubsec:implement}
\noindent\textbf{Rendering settings.} For each iteration, we randomly select 12 camera poses with an 80\% probability of rendering normal maps and a 20\% probability of rendering colored images. The field of view (fovy) range is randomly sampled between 30 and 45 degrees, while the azimuth angles will be discussed in the following paragraph. The initial spherical radius is 2.0, and the camera distance is randomly sampled between 1.5 and 2.0. We did not use soft shading, as we found that it significantly slows down the training process. The final rendering resolutions are set to $64\times64$ and $512\times512$ for the coarse and fine stages, respectively.

\noindent\textbf{Modified batch uniform azimuth sampling.} We observe that vanilla batch uniform azimuth sampling is not suitable for view-dependent prompting selected within $\{$\texttt{front}, \texttt{side}, \texttt{back}$\}$ when the batch size is large (\eg, 12). This can lead to the Janus face issue, as the same guidance is shared between different azimuths. We empirically found that splitting the azimuth range into $4$ parts according to the prompt and uniform sampling of azimuths within each range can alleviate this problem.

\noindent\textbf{Annealed Classifier-free Guidance.}
We empirically find that reasonably large CFG weight leads to better details, which accords with the empirical results in DreamFusion~\cite{poole2022dreamfusion}. The optimization-based generation framework has a single-step sampling approach that is distinct from the multi-step sampling used in image generation. 
This single-step sampling framework requires a typically larger CFG to emphasize the details that appear in a single generation, whereas multi-step sampling is able to accumulate details many times over, thus allowing the use of smaller CFG. This framework also differs from other frameworks that require training of LoRA~\cite{wang2023prolificdreamer}, as LoRA plays a similar role to the high CFG values, \ie, providing better orientation in the single-step generation. However, oversaturation problems can occur in the generated images if the CFG values are too high for small time steps. Therefore, we suggest that the CFG value should also vary with the time step schedule. In other words, it should become progressively smaller. In practice, we linearly decrease the CFG value from $50$ to $20$ with increasing iteration.

\noindent\textbf{Evaluation Settings.} For quantitative evaluation, we measure the CLIP R-precision~\cite{radford2021learning} following the practice of DreamFusion~\cite{poole2022dreamfusion}. We compare with DreamFusion~\cite{poole2022dreamfusion}, Magic3D~\cite{lin2023magic3d} and ProlificDreamer~\cite{wang2023prolificdreamer} using 40 randomly selected text prompts from the DreamFusion gallery~\footnote{https://dreamfusion3d.github.io/gallery.html}. The prompts used in our experiments are listed in \cref{tab:zoo}. We then render $120$ views with uniformly sampled azimuth angles and calculate the CLIP R-precision based on each rendered image, and we average the different views for the final metric.

\subsection{Additional Results}
\label{appsubsec:qualitative}
\begin{table}[!t]
\centering
\begin{tabular}{c|c|c|c}
{(a)}   & {(b)}  & {(c)} & {(d) Ours} \\
\hline
0.319 & 0.325 & 0.340 & \textbf{0.348} \\
\end{tabular}
\vspace{-6pt}
\caption{CLIP R-precision of the ablation study. (a) random time step schedule; (b) predetermined time step
schedule; (c) random noise in each iteration; (d) our proposed configuration.}
\vspace{-9pt}
\label{tab:ab}
\end{table}
\begin{table}[!t]
\centering
\resizebox{1.0\linewidth}{!}{%
\begin{tabular}{c|c|c|c|c}
 & DF     & M3D    & PD     & Ours             \\
\hline
User Prefer. Rate ($\uparrow$) & 15.0\% & 5.4\% & 15.0\% & \textbf{64.6\%} \\
GPT4-V Avg. Rank ($\downarrow$) & 3.0 & 3.8 & 1.8 & \textbf{1.4} \\
\end{tabular}
}
\vspace{-6pt}
\caption{\textbf{User/AI study.} DF: DreamFusion; M3D: Magic3D; PD: ProlificDreamer.}
\vspace{-15pt}
\label{tab:study}
\end{table}

More qualitative results for both the coarse and fine stages are shown in Fig.~\ref{fig:app-ours}. We notice that our Consistent3D can generate robust geometry in the coarse stage and then seemly enhance high-frequency and sophisticated details in the fine stage. 
We conduct the ablation study quantitatively in \cref{tab:ab}, the results also support the  superiority of our Consistent3D over others. 

We also conduct a user study in Tab.~\ref{tab:study} where $50$ users select best from multiple choices on $30$ generated 3D assets. Moreover, GPT4-V is used to evaluate and rank the generated 3D assets from aesthetic appeal, multi-view consistency and alignment with text prompt. \cref{tab:study} demonstrates the superiority of the results generated by our Consistent3D from both human and large multi-modal AI perspectives.

\subsection{Fast Generation with 3D Gaussian Splatting}
\label{appsubsec:gaussian}
Our Consistent3D with Consistency Distillation Sampling is a general text-to-3D generation framework that can be used to create a variety of 3D representations, including 3D Gaussian Splatting~\cite{kerbl20233d}. As demonstrated in Fig.~\ref{fig:app-gaussian}, our Consistent3D is capable of producing high-fidelity 3D models with intricate details in 15 minutes, which vividly showcases the potential of CDS among different 3D representations (NeRF, Mesh, 3D Gaussian Splatting).

\section{Theoretical Proof}
\label{appsec:thm}
\setcounter{theorem}{0}
\begin{theorem}
\label{thm:proof}
Assume that the diffusion model $D_{\phi}(\cdot)$ satisfies the Lipschitz condition. Define $\Delta := \sup |t_1 - t_2| $.
For any given camera pose $\pi$, if convergence is achieved according to Eq.~\eqref{eq:cds}, then there exists a corresponding real image $\bbx^* \sim p_{\text{data}}(\bbx)$ such that
\begin{equation}
\| \bbx_{\pi} - \bbx^* \|_2 = \mathcal{O}(\Delta),
\end{equation}
where $\bbx_{\pi} = g(\bbt, \pi)$ denotes the rendered image for pose $\pi$.
\end{theorem}
\begin{proof}
From $\mathcal{L}_{\text{CDS}}(\bbt; \pi) = 0$, for any given $\pi$ and $T \geq t_n > t_{n+1} \geq 0$, it satisfied that
\begin{equation}
D_{\phi}(\bbx_{t_n}, t_n, y) \equiv D_{\phi}(\hat{\bbx}_{t_{n+1}}, t_{n+1}, y).
\end{equation}
Let $\boldsymbol{e}_n$ represent the error at $t_n$, which is defined as:
\begin{equation}
    \boldsymbol{e}_n := D_{\phi}(\bbx_{t_n}, t_n) - \bbx^*.
\end{equation}
We can derive the error at $t_{n+1}$ given the error at $t_n$:
\begin{align*}
    \boldsymbol{e}_{n} &= D_{\phi}(\bbx_{t_{n}}, t_{n}) - \bbx^* \\
    &= D_{\phi}(\hat{\bbx}_{t_{n+1}}, t_{n+1}) - \bbx^* \\
    &= D_{\phi}(\hat{\bbx}_{t_{n+1}}, t_{n+1}) - D_{\phi}(\bbx_{t_{n+1}}, t_{n+1}) \\
    & \quad + D_{\phi}(\bbx_{t_{n+1}}, t_{n+1}) - \bbx^* \\
    &= D_{\phi}(\hat{\bbx}_{t_{n+1}}, t_{n+1}) - D_{\phi}(\bbx_{t_{n+1}}, t_{n+1}) + \boldsymbol{e}_{n+1}.
\end{align*}
According to the Lipschitz condition, we can further derive
\begin{align*}
    \|\boldsymbol{e}_{n}\| &\leq \| D_{\phi}(\hat{\bbx}_{t_{n+1}}, t_{n+1}) - D_{\phi}(\bbx_{t_{n+1}}, t_{n+1}) \| + \| \boldsymbol{e}_{n+1} \| \\
    &\leq L \|\hat{\bbx}_{t_{n+1}} - \bbx_{t_{n+1}}\| + \| \boldsymbol{e}_{n+1} \| \\ 
    &\overset{(i)}{=} \| \boldsymbol{e}_{n+1} \| + \mathcal{O}((t_n - t_{n+1})^2),
\end{align*}
where $(i)$ hold according to the local error of Euler solver. 
Therefore, we can derive the error recursively:
\begin{align*}
    \| \boldsymbol{e}_0 \| &\leq \sum_{k=0}^{N-1} \mathcal{O}((t_k - t_{k+1})^2) \\
    &\leq \sum_{k=0}^{N-1} (t_k - t_{k+1}) \mathcal{O}(\Delta) \\
    &= \mathcal{O}(\Delta).
\end{align*}
\end{proof}

\begin{remark}
\cref{thm:proof} not only affirms the realistic rendering capabilities of CDS for any camera pose $\pi$, but also highlights that CDS achieves an error bound of $\mathcal{O}(\Delta)$, which is the same as multi-step sampling approaches~\cite{DDIM, tang2023dreamgaussian}. This shows that CDS is capable of achieving the same accuracy as multi-step approaches in a single-step generative framework, thus demonstrating its efficiency and broad applicability for optimization-based generation.
\end{remark}
\section{Limitations}
\label{appsec:lim}

Our approach relies on pre-trained diffusion models without 3D priors, and it may sometimes produce less than satisfactory results, especially in complex 3D modeling scenarios. 
Additionally, pre-trained models may unintentionally transfer unwanted bias from their original training data and parameters into the generated 3D models.

These challenges highlight two critical areas for future research and development in 3D generation. First, there is a pressing need to develop generative models that incorporate robust 3D-centric training. Such models would be better equipped to handle the complexities and nuances inherent in 3D structures. Second, it is essential to devise strategies that effectively identify and neutralize biases transferred from pre-trained models. Addressing these issues will not only improve the accuracy and reliability of 3D generation, but will also ensure the ethical integrity and fairness of the generative process.

\begin{table*}[!t]
\centering
\begin{tabular}{|c|l|}
\hline
\textbf{ID} & \textbf{Prompt} \\
\hline
1 & A DSLR photo of a squirrel dressed like a clown \\
2 & A DSLR photo of a tiger made out of yarn \\
3 & A DSLR photo of the Imperial State Crown of England \\
4 & A beagle in a detective's outfit \\
5 & A blue motorcycle \\
6 & A blue poison-dart frog sitting on a water lily \\
7 & A cat with a mullet \\
8 & A ceramic lion \\
9 & A highly detailed sand castle \\
10 & A lemur taking notes in a journal \\
11 & A panda rowing a boat \\
12 & A panda wearing a necktie and sitting in an office chair \\
13 & A photo of a skiing penguin wearing a puffy jacket, highly realistic DSLR photo \\
14 & A photo of a tiger dressed as a doctor, highly realistic DSLR photo \\
15 & A photo of the Ironman, highly realistic DSLR photo \\
16 & A rabbit, animated movie character, high detail 3D model \\
17 & A silver platter piled high with fruits \\
18 & A squirrel dressed like Henry VIII king of England \\
19 & A tarantula, highly detailed \\
20 & A tiger wearing a tuxedo \\
21 & A wide-angle DSLR photo of a squirrel in samurai armor wielding a katana \\
22 & A zoomed-out DSLR photo of a 3D model of an adorable cottage with a thatched roof \\
23 & A zoomed-out DSLR photo of a human skeleton relaxing in a lounge chair \\
24 & A zoomed-out DSLR photo of a model of a house in Tudor style \\
25 & A zoomed-out DSLR photo of a panda throwing wads of cash into the air \\
26 & An astronaut is riding a horse \\
27 & The Great Wall, aerial view \\
28 & Michelangelo style statue of dog reading news on a cellphone \\
29 & A DSLR photo of a baby dragon hatching out of a stone egg \\
30 & A DSLR photo of a bald eagle \\
31 & A DSLR photo of a bear dressed in medieval armor \\
32 & A DSLR photo of a humanoid robot holding a human brain \\
33 & A DSLR photo of a peacock on a surfboard \\
34 & A DSLR photo of a red pickup truck driving across a stream \\
35 & A DSLR photo of a Shiba Inu playing golf wearing tartan golf clothes and hat \\
36 & A 20-sided die made out of glass. \\
37 & A DSLR photo of Mount Fuji, aerial view. \\
38 & An old vintage car. \\
39 & A delicious hamburger. \\
40 & A cute steampunk elephant. \\
\hline
\end{tabular}
\caption{\textbf{Prompt library} for quantitative results.}
\label{tab:zoo}
\end{table*}

\end{document}